\documentclass{article}

\usepackage{PRIMEarxiv}

\usepackage[utf8]{inputenc} 
\usepackage[T1]{fontenc}    
\usepackage{hyperref}       
\usepackage{url}            
\usepackage{booktabs}       
\usepackage{amsfonts}       
\usepackage{nicefrac}       
\usepackage{microtype}      
\usepackage{lipsum}
\usepackage{fancyhdr}       
\usepackage{graphicx}       
\graphicspath{{media/}}     
\usepackage{tikz}
\usepackage{multirow}
\usepackage{amsmath}
\usepackage{amsthm}
\usepackage{enumitem}
\usepackage{comment}
\usepackage{pifont}
\usepackage{dsfont}
\usetikzlibrary{shapes.geometric,through, decorations.pathmorphing,decorations.pathreplacing, arrows.meta,quotes,mindmap,shapes.symbols,shapes.arrows,automata,angles,3d,trees,shadows,automata,arrows}

\excludecomment{proof}

\newcommand{\true}{\cmark}
\newcommand{\false}{\xmark}
\newcommand{\cmark}{\ding{51}}%
\newcommand{\xmark}{\ding{55}}
\newcommand{\AS}{\mathbf{AS}}
\newcommand{\A}{\mathcal{A}}
\newcommand{\C}{\mathcal{C}}
\newcommand{\Str}{\mathtt{Str}}
\newcommand{\Sh}{\mathtt{Sh}}

\newcommand{\Att}{Att}

\newcommand{\nameImp}{\mathtt{Imp}}
\newcommand{\nameImpDV}{\mathtt{ImpDV}}
\newcommand{\nameImpSI}{\mathtt{ImpSI}}
\newcommand{\imp}[2]{\nameImp^{#1}_{#2}}
\newcommand{\impDV}[2]{\nameImpDV^{#1}_{#2}}
\newcommand{\impSI}[2]{\nameImpSI^{#1}_{#2}}

\newcommand{\Sem}{\sigma}
\newcommand{\nameHbs}{\mathtt{Hbs}}
\newcommand{\nameMax}{\mathtt{Max}}
\newcommand{\nameCar}{\mathtt{Car}}
\newcommand{\nameCS}{\mathtt{CS}}
\newcommand{\Hbs}{\sigma^{\nameHbs}}
\newcommand{\Max}{\sigma^{\nameMax}}
\newcommand{\Car}{\sigma^{\nameCar}}
\newcommand{\CS}{\sigma^{\nameCS}}

\newcommand{\setAttackArg}{Arg^{-}}
\newcommand{\setAttackAtt}{Att^{-}}
\newcommand{\compOpArg}[1]{\ominus^{\A}_{#1}}
\newcommand{\compOpAtt}[1]{\ominus^{\C}}

\newtheorem{definition}{Definition}
\newtheorem{example}{Example}
\newtheorem{conjecture}{Conjecture}
\newtheorem{property}{Property}
\newtheorem{principle}{Principle}

\newtheorem{theorem}{Theorem}

\newtheorem{proposition}[theorem]{Proposition}

\title{Impact Measures for Gradual Argumentation Semantics}

\author{
  Caren Al Anaissy\\
 CRIL Université d’Artois \& CNRS\\
 France \\
  \texttt{alanaissy@cril.fr} \\
   \And
  Jérôme Delobelle \\
  Universit\'e Paris Cit\'e, LIPADE, F-75006\\
  France\\
  \texttt{jerome.delobelle@u-paris.fr} \\
  \And
  Srdjan Vesic \\
  CRIL CNRS Univ. Artois \\
  France\\
  \texttt{vesic@cril.fr} \\
  \And
  Bruno Yun \\
  Universite Claude Bernard Lyon 1, CNRS, Ecole Centrale de Lyon, INSA Lyon,\\
  Université Lumière Lyon 2, LIRIS, UMR5205\\
  France \\
  \texttt{bruno.yun@univ-lyon1.fr} \\
}

\begin{document}
\maketitle

\begin{abstract}
Argumentation is a formalism allowing to reason with contradictory information by modeling arguments and their interactions.
There are now an increasing number of gradual semantics and impact measures that have emerged to facilitate the interpretation of their outcomes. An impact measure assesses, for each argument, the impact of other arguments on its score.
In this paper, we refine an existing impact measure from Delobelle and Villata and introduce a new impact measure rooted in Shapley values. We introduce several principles to evaluate those two impact measures w.r.t.\ some well-known gradual semantics.
This comprehensive analysis provides deeper insights into their functionality and desirability.\\
\end{abstract}


\section{Introduction}

Computational argumentation theory stands as an important domain of artificial intelligence (AI), especially in knowledge representation and reasoning. 
It is used -- among others -- for inferring conclusions in decision making problems~\cite{amgoud2009using,muller2012argumentation,gao2016argumentation} and for resolving conflicts of opinion in persuasion and negotiation dialogues ~\cite{rosenfeld2016strategical,amgoud2007unified}. 
It represents knowledge in \textit{argumentation graphs} with arguments as nodes and a binary attack relation for conflicts between pieces of information.
Different semantics can then be applied on those graphs to obtain rational conclusions. In this paper, we focus on \textit{gradual semantics} which evaluate and score each argument w.r.t.\ ``how much'' it is attacked by other arguments.

Recently, explainable AI has garnered more attention for its ability to provide transparency and enhance the understandability of AI-based models and algorithms.
Among the existing notions of explainability in the literature~\cite{miller2019explanation,vilone2021notions}, two notions seem pivotal to paving the way toward explaining the outcomes of argument evaluation: \textit{Causality}~\cite{ha2018designing} and \textit{Feature Attribution}~\cite{ribeiro2016should,lundberg2017unified}. 
Causality refers to the ability to explain the link between what is introduced as input and what results as output. 
Feature attribution means assessing the contribution of input features (e.g., the age or the blood pressure) to the output (e.g., the probability of getting sick) by assigning attribution scores to each feature, avoiding the need to explore the model's internal mechanisms. 

Hence, we argue that -- working toward making argumentation interpretable -- it is essential that the evaluation of an argument can be explained by highlighting the attribution of each argument’s interconnected network to its final evaluation.
Namely, this is achieved by providing to the user an impact measure that returns the contribution of different arguments to the argument’s final evaluation. 

Few impact measures have been defined in the literature~\cite{potyka2018continuous,DBLP:conf/ecsqaru/DelobelleV19,DBLP:conf/sgai/HimeurYBC21,yin2023argument}. 
We believe that there is still the need to explore other impact measures that have not been yet considered for gradual semantics.
For this reason, this paper addresses the need for enhanced impact measures, bringing modifications to existing functions while introducing novel ones. 
Our study navigates the intricate landscape by first revising the impact measure defined in \cite{DBLP:conf/ecsqaru/DelobelleV19}. 
Additionally, we define a complementary impact measure based on the Shapley Contribution Measure~\cite{shapley_value_1953,amgoud_measuring_2017-1}. These impact measures can be used to assess the impact of an argument (or set of arguments) on a given argument for any existing gradual semantics, clarifying and explaining their contributions on the score of another argument.  

The primary contributions of this paper culminate in three facets. 
First, the enhancement of an existing impact measure and the definition of a novel impact measure based on Shapley Contribution Measure. 
Second, the introduction of nine principles for evaluating each (impact, semantics) pair and a full analysis of two impact measures under some well-known gradual semantics. 
Lastly, the implementation of an  online prototype platform where users can input their argumentation graphs, compute the acceptability degrees for a gradual semantics, and obtain the output of our new impact measures.

We start with a comprehensive review of existing literature and methodologies (Section \ref{sec:background}),  transitioning into the modification and definition of impact measures (Section \ref{sec:impact-grad}). Subsequently, we analyse our findings through the defined principles by analysing their satisfiability by different (impact, semantics) pairs (Section \ref{sec:desirable}). The paper concludes by summarising the implications of our work and outlining future directions in this evolving landscape (Section \ref{sec:related-discussion}).

\section{Background Knowledge}

\label{sec:background}

We introduce the argumentation concepts used in this paper.

\begin{definition}[Argumentation Framework]
An argumentation framework (AF) is $\AS = (\A,\C)$, where $\A$ is a finite set of arguments and $\C \subseteq \A \times \A$ is a set of attacks between arguments. 
\end{definition}

The \textit{set of all direct attackers} of $x \in \A$ will be denoted as $\Att(x) = \{ y\in \A \mid (y,x) \in \C\}$. 
Given two AFs $\AS = (\A,\C)$ and $\AS' = (\A',\C')$, $\AS \oplus \AS'$ is the AF $(\A \cup \A', \C \cup \C')$.
For any AF $\AS = (\A,\C)$ and $X \subseteq \A$, $\AS\vert_{X} = (X, \C \cap (X \times X)).$
The \textit{set of external attackers of X} is $ \setAttackArg(X) = \{y \in \A \backslash X\ |\ \exists \ x \in X\ s.t.\ (y,x) \in \C \}$.
The \textit{set of external attacks to X} is the set of attacks from an argument in $\setAttackArg(X)$ to an argument in $X$ and is formally defined as follows: $\setAttackAtt(X) = \{(y,x) \in \C\ |\ y \in \setAttackArg(X)\ and\ x \in X\}$.
Let $x,y \in \A$, a \textit{path} from $y$ to $x$ is a sequence $\langle x_{0},\dots,x_{n} \rangle$ of arguments such that $x_{0} = y$, $x_{n} = x$ and $\forall i$ s.t. $0 \le i < n, (x_{i},x_{i+1}) \in \C$.\\
The attack structure of an argument $x$ is a set of arguments that contains $x$ and all the (direct or indirect) attackers and defenders of $x$.

\begin{definition}[Attack Structure]
\label{attackstructure}
Let $\AS = (\A,\C)$ be an AF and $x \in \A$. The attack structure of $x$ in $\AS$ is $\Str_\AS(x) = \{ x \} \cup \{ y \in \A \mid$ there is a path from $y$ to $x\}$.  
\end{definition}

\begin{example}
\label{ex:1}

Consider the AF $\AS = (\A, \C)$ represented in Figure \ref{fig:AF}. 
We have $\Att(a_4) = \{a_3, a_5, a_8\}$, 
$\setAttackArg(\{a_{8}, a_{10}\}) = \{a_9\}$, 
$\setAttackAtt(\{a_{8}, a_{10}\}) = \{(a_{9},a_{8}),(a_{9},a_{10})\}$,
$\Str_\AS(a_3) = \{a_1,a_2,a_3\}$ and $\Str_\AS(a_4) = \{a_1,a_2,a_3,a_4,a_5,a_6,a_8,a_9,a_{10}\}$.

\begin{figure}[h]
    \begin{center}
    	\begin{tikzpicture}[->,shorten >=0.5pt,auto,node distance=1.2cm, thick,
                    main node/.style={circle,draw,inner sep=2pt,font=\small},
                    second node/.style={circle,draw,inner sep=1pt,font=\small}]
    	\node[main node] (1) {$a_1$};
    	\node[main node] (2) [left of=1] {$a_2$};
    	\node[main node] (3) [below of=2] {$a_3$};
    	\node[main node] (5) [left of=3] {$a_5$};
    	\node[main node] (4) [below of=5] {$a_4$};
    	\node[main node] (6) [above of=5] {$a_6$};
    	\node[main node] (8) [left of=5] {$a_8$};
    	\node[main node] (7) [left of=8] {$a_7$};
    	\node[main node] (9) [above of=8] {$a_9$};
    	\node[second node] (10) [left of=9] {$a_{10}$};
    	\node[second node] (11) [below right of=3] {$a_{11}$};
    	
    	\draw[->,>=latex] (1) to (3);
    	\draw[->,>=latex] (2) to (3);
    	\draw[->,>=latex] (3) to (4);
    	\draw[->,>=latex] (6) to (5);
    	\draw[->,>=latex] (5) to (4);
    	\draw[->,>=latex] (8) to (4);
    	\draw[->,>=latex] (9) to (8);
    	\draw[->,>=latex] (8) to (7);
    	\draw[->,>=latex] (9) to [bend right=15] (10);
    	\draw[->,>=latex] (10) to [bend right=15] (9);
    	\draw[->,>=latex] (1) to [bend right=15] (2);
    	\draw[->,>=latex] (2) to [bend right=15] (1);
    	\end{tikzpicture}
    \end{center}
    \caption{\label{fig:AF}Graphical representation of an argumentation framework.}
\end{figure}
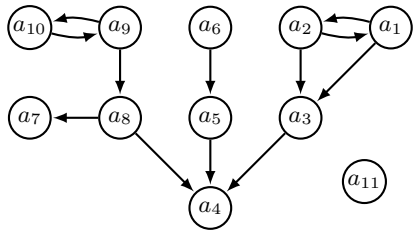


\end{example}

The usual Dung's semantics~\cite{dung_acceptability_1995} induces a two-levels acceptability of arguments (inside or outside of one or all extensions). Gradual semantics (and ranking-based semantics) have been proposed as a more fine-grained approach to argument acceptability~\cite{DBLP:conf/sum/AmgoudB13a,BonzonDKM23}. 
These semantics use a weighting to assign to each argument in the argumentation framework a score, called (acceptability) degree.

\begin{definition}[Weighting]
    A weighting on a set $\mathcal{X}$ is a function from $\mathcal{X}$ to $[0,1]$.
\end{definition}

\begin{definition}[Gradual semantics] \label{def:gradualSemantics}
    A gradual semantics is a function $\sigma$ which associates to each $\AS = (\A,\C)$, a weighting $\sigma_\AS: \A \to [0,1]$ on $\A$. $\sigma_\AS(a)$ is called the degree of $a$. 
\end{definition}

Let us now recall some well-known gradual semantics studied in the literature~\cite{amgoud_evaluation_2022,besnard_logic-based_2001,DBLP:conf/ijcai/OrenYVB22}.\\
The h-categoriser semantics ($\nameHbs$) assigns a value to each argument by taking into account the sum of degrees of its attackers, which themselves take into account the degree of their attackers.

\begin{definition}[h-categoriser]\label{h-categoriser}
The h-categoriser semantics is a gradual semantics $\Hbs$ s.t. for any $\AS = (\A,\C)$ and $a \in \A$:

$$\Hbs_\AS(a) = \frac{1}{1+ \Sigma_{b \in \Att(a)} \Hbs_\AS(b)}$$
\end{definition}

The Card-based semantics ($\nameCar$) favors the number of attackers over their quality. 
This semantics is based on a recursive function which assigns a score to each argument on the basis of the number of its direct attackers and their degrees.

\begin{definition}[Card-based]\label{card-based}
    The Card-based semantics is a gradual semantics $\Car$ s.t. for any $\AS = (\A,\C)$ and $a \in \A$:

$$\Car_\AS(a) = \frac{1}{1+ |\Att(a)| + \frac{\Sigma_{b \in \Att(a)} \Car_\AS(b)}{|\Att(a)|}}$$
\end{definition}

The Max-based semantics ($\nameMax$) favors the quality of attackers over their number.
The degree of an argument is based on the degree of its strongest direct attacker.

\begin{definition}[Max-based]\label{max-based}
    The Max-based semantics is a gradual semantics $\Max$ s.t. any $\AS = (\A,\C)$ and $a \in \A$:

$$\Max_\AS(a) = \frac{1}{1+ \max_{b \in \Att(a)} \Max_\AS(b)}$$
\end{definition}

The last gradual semantics we study is the counting semantics ($\nameCS$) \cite{PLZL15}. 
It assigns a value to each argument by counting the number of their respective attackers and defenders.
An AF is considered as a dialogue game between the proponents of a given argument $x$ (i.e., the defenders of $x$) and the opponents of $x$ (i.e., the attackers of $x$). 
Thus, the degree of an argument is greater if it has many arguments from proponents and few arguments from opponents.
Formally, they first convert a given $AF$ into a matrix $M_{n \times n}$ (where $n$ is the number of arguments in $AF$) which corresponds to the adjacency matrix of $AF$ (as an $AF$ is a directed graph). 
The matrix product of $k$ copies of $M$, denoted by $M^{k}$, represents, for all the arguments in $AF$, the number of defenders (if $k$ is even) or attackers (if $k$ is odd) situated at the beginning of a path of length $k$.
Finally, a normalization factor $N$ (e.g., the matrix infinite norm) is applied to $M$ in order to guarantee the convergence, and a damping factor $\alpha$ is used to have a more refined treatment on different length of attacker and defenders (i.e., shorter attacker/defender lines are preferred).

\begin{definition}[Counting model]
	Let $\AS = (\A,\C)$ be an argumentation framework with $\A = \{x_1,\dots,x_n\}$, $\alpha \in\ (0,1)$ be a damping factor and $k \in \mathds{N}$.
	The $n$-dimensional column vector $v$ over $\A$ at step $k$ is defined by,\\
	\centerline{$v^{k}_{\alpha} = \sum\limits_{i=0}^{k} (-1)^{i} \alpha^i \tilde{M}^i \mathcal{I}$}
	where $\tilde{M} =M/N$ is the normalized matrix with $N$ as normalization factor and $\mathcal{I}$ as the $n$-dimensional column vector containing only 1s.\\ 
	The counting model of $\AS$ is $v_{\alpha} = \lim\limits_{k \rightarrow +\infty} v^{k}_{\alpha} $.
	The degree of $x_i \in \A$ is the $i^{th}$ component of $v_{\alpha}$, denoted by $\CS_\AS(x_{i})$.
\end{definition}

Together with the introduction of these semantics, a number of properties have been defined for gradual semantics to evaluate their behaviours (see \cite{amgoud_evaluation_2022,BeuselinckDV23} for an overview). 
Two of the most well-known are called Independence and Directionality. 
The former ensures that the calculation of the acceptability degrees in two disconnected AFs should be independent while the latter ensures that the acceptability degree of an argument should only rely on the arguments with a directed path to it.

 \begin{property}[Independence~\cite{amgoud_evaluation_2022}]
A semantics $\sigma$ satisfies Independence iff for any two AFs $\AS = (\A,\C)$, $\AS' = (\A',\C')$, where $\A \cap \A' = \emptyset$, for every $y \in \A, \sigma_{\AS}(y) = \sigma_{\AS \oplus \AS'}(y)$.
 \end{property}

  \begin{property}[Directionality~\cite{amgoud_evaluation_2022}]
A semantics $\sigma$ satisfies Directionality iff for any AF $\AS = (\A,\C)$, $\AS' = (\A,\C \cup \{(b,x)\})$, for every $y \in \A$ such that there is no path from $x$ to $y$, $\sigma_{\AS}(y) = \sigma_{\AS'}(y)$.
 \end{property}

To the best of our knowledge, the counting semantics is the only one in the literature not to satisfy these two properties. 
However, we have chosen to include this semantics in our study because the approach used differs from the other gradual semantics studied ($\nameHbs$, $\nameCar$ and $\nameMax$), which all satisfy Independence and Directionality.

 \begin{proposition}
     The counting semantics does not satisfy the Independence and Directionality properties.
     \label{prop:CS_ind_directionality}
 \end{proposition}

 \begin{proof}
\input{Proofs/Prop_CS_Independence_Directionality}
\end{proof}

\section{Impact for Gradual Semantics}
\label{sec:impact-grad}

An impact measure is a function that informs on how a set of arguments ``impacts'' the score of a specific argument. 
In this paper, it returns a value within the interval $[-1,1]$, reflecting the impact of the set and its overall polarity (negative, positive or neutral).

\begin{definition}[Impact measure]
    Let $\AS = (\A,\C)$ be an AF and $\Sem$ be a gradual semantics. An impact measure $\nameImp$ takes as input $\AS$ and $\sigma$ and returns a function $\imp{\sigma}{\AS}: 2^\A \times \A \to [-1,1]$. For any $X \subseteq \A$, $y \in \A$, $\imp{\sigma}{\AS}(X,y)$ is the impact of $X$ on $y$ (in $\AS$ w.r.t.\ semantics $\sigma$).
\end{definition}

In sub-section \ref{sec:delobelle}, we explain the drawbacks of the impact measure by Delobelle and Villata \cite{DBLP:conf/ecsqaru/DelobelleV19} and introduce a revised version. We also motivate and showcase a new Shapley-based impact measure in sub-section \ref{sec:shap-imppact}.

Note that these two impact measures have been implemented and deployed in an online platform prototype, accessible at \url{https://impact-gradual-semantics.vercel.app/}.

\subsection{Revised version of impact measure from Delobelle and Villata}
\label{sec:delobelle}

In~\cite{DBLP:conf/ecsqaru/DelobelleV19}, the impact of an argument (or a set of arguments) on a target argument can be measured by computing the difference of acceptability degree of the target argument when this element exists and when it is deleted.
To capture this notion of deletion, two deletion operators need to be defined.
The \textit{argument deletion operator} $\compOpArg{}$ aims to delete a set of arguments from a given argumentation framework.
These changes have also a direct impact on the set of attacks because the attacks directly related to the deleted arguments (attacking as well as attacked) are automatically deleted too.\footnote{Note that for $\compOpArg{}$, it is necessary to specify the argument $y$ for which the degree will be measured, in order to avoid removing it from the AF if it belongs to the set of arguments whose impact on $y$ is to be measured.}
The \textit{attack deletion operator} $\compOpAtt{}$ focuses only on the removal of a set of attacks from the initial argumentation framework, thus keeping the same set of arguments.
%

\begin{definition}[Deletion operator]
\label{deletion_operator}
Let $\AS = (\A ,\C)$ be an AF, $X \subseteq \A$ be a set of arguments, $R \subseteq \C$ be a set of attacks and $y \in \A$.
The \textit{argument deletion operator} $\compOpArg{}$ is defined as $\AS \compOpArg{y} X = (\A',\C')$, where
    \begin{itemize}
        \item $\A' = \A \backslash (X \backslash \{y\})$;
	\item $\C' = \{ (x,z) \in \C \mid x \in \A \backslash X, z \in \A \backslash X \}$.
    \end{itemize}
The \textit{attack deletion operator} $\compOpAtt{}$ is defined as $\AS \compOpAtt{} R = (\A,\C'')$, where $\C'' = \C \backslash R$.
\end{definition}

In~\cite{DBLP:conf/ecsqaru/DelobelleV19}, to compute the impact of any set of arguments $X$ on an argument $y$, it is proposed to consider the degree of acceptability of $y$ when the direct attackers of $X$ are removed (i.e., when the arguments in $X$ are the strongest) from which the degree of acceptability of $y$ is deducted when all the arguments of $X$ are removed. 


\begin{definition}[\cite{DBLP:conf/ecsqaru/DelobelleV19}] 
\label{def:impactDVold}
Let $\AS = (\A,\C)$ be an AF, $y \in \A$ and $X \subseteq \A$. Let $\Sem$ be a gradual semantics. 
$\impDV{}{}$ is defined as follows:
$$\impDV{\Sem}{\AS}(X,y) =  \sigma_{\AS \compOpArg{y} \setAttackArg(X)}(y)  -  \sigma_{\AS \compOpArg{y} X}(y) $$
\end{definition}

Although this definition works well in general, there are particular cases (e.g., when self-attacks are allowed or when the direct attackers of $X$ impact $y$ via other paths) where the result does not correspond to what is expected.
For example, if we take $\AS_{s} = (\{a\},\{(a,a)\})$, applying Definition \ref{def:impactDVold} with the h-categoriser semantics (any other semantics gives the same result), we obtain $\impDV{\nameHbs}{\AS_{s}}(\{a\},a) = 0$ because it is not possible to delete the argument whose impact we want to evaluate (otherwise we can no longer calculate its degree with the gradual semantics).
This result can be considered counter-intuitive because the degree of $a$ is not maximal and the only argument who can have an impact on its degree is $a$ itself.

We propose a revised version of $\nameImpDV$ which takes into account this case and other problems while retaining the idea of the proposed approach.
Thus, instead of removing all the direct attackers from $X$ (i.e., using $\compOpArg{}$), we remove the direct attacks on each argument of $X$ (i.e., using $\compOpAtt{}$).

\begin{definition}[Revised version of $\nameImpDV$] \label{def:impact}
Let $\AS = (\A,\C)$ be an AF, $y \in \A$ and $X \subseteq \A$. Let $\Sem$ be a gradual semantics. 
$\impDV{}{}$ is defined as follows:
$$\impDV{\Sem}{\AS}(X,y) =  \sigma_{\AS \compOpAtt{y} \setAttackAtt(X)}(y)  -  \sigma_{\AS \compOpArg{y} X}(y) $$
\end{definition}

Regarding the example of the self-attacking argument, the problem mentioned earlier is solved with the revised version because the self-attack is removed from the right-hand side of the formula. Thus, we have $\impDV{\nameHbs}{\AS_{s}}(\{a\},a)$ =  $\Hbs_{\AS_{s} \compOpAtt{a} \{\emptyset\}}(a)  -  \Hbs_{\AS_{s} \compOpArg{a} \{a\}}(a) \simeq 0.618 - 1 = -0.382$.

In the following sections, $\impDV{}{}$ will refer to that defined in Definition \ref{def:impact}.

\color{black}

\begin{theorem} \label{theorem:interval_impDV}
Let $\sigma$ be a gradual semantics and $\AS = (\A,\C)$ be an AF. 
$\impDV{\sigma}{\AS}$ is an impact measure.

\end{theorem}

\begin{example}
Let $\AS$ be the AF depicted in Figure \ref{fig:AF}. We display the values of the revised impact of a set $X$ on an argument $y$ in Table \ref{tab:decomp-ex-hcat}.
Using the same examples of Table \ref{tab:decomp-ex-hcat}, the original impact measure of~\cite{DBLP:conf/ecsqaru/DelobelleV19} would yield the same values with the exception of the impact of $\{a_1 \}$ on $a_4$ with a value of $0$.
To understand this difference, let us first note that we need to compute the score of $a_4$ when $a_2$ is removed (because $\setAttackArg(\{a_1\}) = \{a_2\}$). 
Since $a_1$ and $a_2$ are symmetrical in terms of attacks received and given (i.e. they attack each other and both attack $a_3$), the impact of $\{a_1\}$ on $a_4$ is equal to 0. As for the revised impact measure, it yields 0.015 because the revised measure solves the problem raised in the scenario where "the direct attackers of X impact y via other paths”. Indeed, removing only the attack from $a_2$ to $a_1$ allows the attacks $(a_1,a_2)$, $(a_2,a_5)$ and $(a_2,a_3)$ to be maintained. Hence the  acceptability degree of $a_4$ when removing $(a_2, a_1)$ is different from the acceptability degree of $a_4$ when removing $a_1$, resulting in a non-null impact value of $\{a_1\}$ on $a_4$.

\begin{table}[t]
    \centering
    \resizebox{\linewidth}{!}{
    \begin{tabular}{cccccc}
         \toprule
         $X$ & $y$ & $\setAttackAtt(X)$ &$\Hbs_{\AS_{1}}(y)$ & $\Hbs_{\AS_{2}}(y)$ & $\impDV{\nameHbs}{\AS}(X,y)$ \\
         \midrule
         $\{a_1\}$ & $a_4$ & $\{(a_2,a_1)\}$ & 0.397 & 0.382 & 0.015 \\
         $\{a_5\}$ & $a_4$ & $\{(a_6,a_5)\}$ & 0.326 & 0.484  &-0.158 \\
         $\{a_8,a_{10}\}$ & $a_4$ & $\{(a_9,a_{10}),(a_9,a_8)\}$ & 0.339 & 0.514 & -0.175 \\
         $\{a_9\}$& $a_4$ & $\{(a_{10},a_9)\}$ &0.409 & 0.339  & 0.07 \\
         \multirow{2}{*}{$\{a_4\}$}& \multirow{2}{*}{$a_5$} & $\{(a_8,a_4),(a_5,a_4),$  &\multirow{2}{*}{0.5} & \multirow{2}{*}{0.5}  &\multirow{2}{*}{0} \\
         & & $(a_3,a_4)\}$ & & & \\
         \bottomrule
    \end{tabular}}
        \caption{Values of $\impDV{\nameHbs}{\AS}(X,y)$ with $\AS_{1} = \AS \compOpAtt{y} \setAttackAtt(X)$ and $\AS_{2} = \AS \compOpArg{y} X$
        .}
    \label{tab:decomp-ex-hcat}
\end{table}  
\end{example}

Note that for $\nameCS$, in order to limit the problem related to its violation of the Independence property (because of the normalisation factor $N$), we have adopted the same approach as \cite{DBLP:conf/ecsqaru/DelobelleV19}, i.e. $N$ is the same for the two subgraphs constructed in Definition \ref{def:impact}.

\subsection{Impact measure based on Shapley value}
\label{sec:shap-imppact}

In an argumentative setting, the Shapley contribution measure was used to calculate the contribution of direct attackers on an argument~\cite{amgoud_measuring_2017-1}.
In this sub-section, we define an impact measure based on this measure. This allows for the calculation of the impact of any set of arguments on another argument. 

Given a gradual semantics that satisfies the Attack Removal Monotonicity property (see Property \ref{prop-attack-removal-monotonicity} below), we recall that the Shapley measure is a function that associates to each attack a number in $[0,1]$ such that for each argument, the loss in its acceptability is equal to the sum of the values of all the attacks toward it.
The general intuition underlying this property is that attacks cannot be beneficial for arguments.

\begin{property}[Attack Removal Monotonicity~\cite{amgoud_measuring_2017-1}]
\label{prop-attack-removal-monotonicity}
A semantics $\sigma$ satisfies Attack Removal Monotonicity iff for any AF $\AS=(\A,\C)$, for every $a \in \A$, for every $R \subseteq \{(x,a) \mid x \in \Att(a) \}$, it holds that $\sigma_\AS(a) \leq \sigma_{\AS \compOpAtt{y} R}(a).$
\end{property}

In \cite{amgoud_measuring_2017-1}, they conjectured that h-categoriser satisfies Property \ref{prop-attack-removal-monotonicity}.  We agree with that conjecture and, from our implementations and experiments, we hypothesise that the max-based and card-based semantics also satisfy Property \ref{prop-attack-removal-monotonicity}. However, due to its nature, we show that the counting semantics does not satisfy Property \ref{prop-attack-removal-monotonicity} (see Example \ref{ex-cs-not-monotone}).

\begin{example}

Consider the two AFs $\AS_{1}$ and $\AS_{2}$ depicted in Figure \ref{fig:cs-monotonicity}. 
We see that $\CS_{\AS_{2}}(a_3) = 0.26$ is strictly higher than $\CS_{\AS_{1}}(a_3) = 0.02$.

\begin{figure}
    \centering
    \begin{tikzpicture}[scale=0.075]
\tikzstyle{every node}+=[inner sep=0pt]
\draw [black] (51.8,-26.1) circle (3);
\draw (51.8,-26.1) node[label={[label distance=8] 270:  \color{blue} \scriptsize  0.02}] {$a_3$};
\draw [black] (39.1,-19.7) circle (3);
\draw (39.1,-19.7) node[label={[label distance=8] 90:  \color{blue} \scriptsize  0.02}] {$a_2$};
\draw [black] (27.8,-26.1) circle (3);
\draw (27.8,-26.1) node[label={[label distance=8] 270:  \color{blue} \scriptsize  1.00}] {$a_1$};
\draw (22.8,-17) node {$\AS_{1}$};
\draw [black] (30.41,-24.62) -- (36.49,-21.18);
\fill [black] (36.49,-21.18) -- (35.55,-21.14) -- (36.04,-22.01);
\draw [black] (30.8,-26.1) -- (48.8,-26.1) ;
\fill [black] (48.8,-26.1) -- (48,-25.6) -- (48,-26.6);
\end{tikzpicture}
\quad
    \begin{tikzpicture}[scale=0.075]
\tikzstyle{every node}+=[inner sep=0pt]
\draw [black] (51.8,-26.1) circle (3);
\draw (51.8,-26.1) node[label={[label distance=8] 270:  \color{blue} \scriptsize  0.26}] {$a_3$};
\draw [black] (39.1,-19.7) circle (3);
\draw (39.1,-19.7) node[label={[label distance=8] 90:  \color{blue} \scriptsize  0.51}] {$a_2$};
\draw [black] (27.8,-26.1) circle (3);
\draw (27.8,-26.1) node[label={[label distance=8] 270:  \color{blue} \scriptsize  1.00}] {$a_1$};
\draw (56.8,-17) node {$\AS_{2}$};
\draw [black] (30.41,-24.62) -- (36.49,-21.18) node[midway,label={[label distance=3] 135:  \color{red} \scriptsize 0.49}]{};
\fill [black] (36.49,-21.18) -- (35.55,-21.14) -- (36.04,-22.01);
\draw [black] (41.78,-21.05) -- (49.12,-24.75) node[midway,label={[label distance=3] 45:  \color{red} \scriptsize -0.11}]{};
\fill [black] (49.12,-24.75) -- (48.63,-23.94) -- (48.18,-24.84);
\draw [black] (30.8,-26.1) -- (48.8,-26.1) node[midway,label={[label distance=3] 270:  \color{red} \scriptsize 0.85}]{};
\fill [black] (48.8,-26.1) -- (48,-25.6) -- (48,-26.6);
\end{tikzpicture}
    \caption{Counter-example for the satisfaction of Property \ref{prop-attack-removal-monotonicity} by the counting semantics. Values in blue represent the acceptability degrees for the two AFs for the counting semantics. Values in red represent the Shapley measure for each of the attacks.}
    \label{fig:cs-monotonicity}
\end{figure}
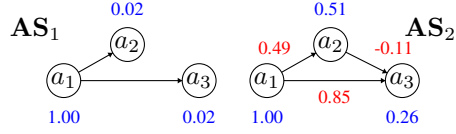
\label{ex-cs-not-monotone}
\end{example}

Let us now formally define the Shapley measure.

\begin{definition}[Extended Shapley measure]
\label{def:shapley_measure}
Let $\AS= (\A,\C)$ be an AF and $\sigma$ be a gradual semantics. The Shapley measure, w.r.t. $\sigma$, is the function $s: \C \to [-1,1]$ such that:
   $$ s((b,a)) = \sum_{X \subseteq Y} \frac{|X|!  (n - |X| - 1)!}{n!} (\sigma_{\AS_2}(a) -\sigma_{\AS_1}(a)),$$  where $Y = \{ (y,a) \mid y \in \Att(a)\} \setminus \{(b,a)\}, n = |\Att(a)|, \AS_1 = \AS \compOpAtt{} X, $ and $\AS_2= \AS \compOpAtt{} (X \cup \{(b,a)\})$.
\end{definition}

As impact measures must account for a variety of gradual semantics, Definition \ref{def:shapley_measure} generalises the Shapley measure of \cite{amgoud_measuring_2017-1} as a function which returns a value in $[-1,1]$ instead of $[0,1]$.
Indeed, since the counting semantics does not satisfy Property \ref{prop-attack-removal-monotonicity}, applying the Shapley measure can result in attacks with negative contributions.
Namely, the attacks with negative contributions are attacks that increase the degree of the target argument. In Figure \ref{fig:cs-monotonicity} (right), the attack $(a_2,a_3)$ increases the degree of $a_3$ by $0.11$ but $(a_1,a_3)$ decreases the degree of $a_3$ by $0.85$, resulting in a degree of $0.26$.
This new extended Shapley measure allows to highlight this surprising behaviour of gradual semantics which was not possible with the old version.
However, note that if a gradual semantics satisfies Attack Removal Monotonicity, then the value returned by the Shapley measure will remain in the interval $[0,1]$.

In the next definition, we define the Shapley-based impact measure based on this extended Shapley measure.
To compute the impact of any set of arguments $X$ on an argument $y$, the Shapley-based impact measure considers, for each argument $x$ in $X$, the difference between the contribution of even paths and odd paths from $x$ to $y$.

\begin{definition}[Shapley-based impact measure]
Let $\AS = (\A,\C)$ be an AF, $a \in \A, X \subseteq \A, \sigma$ be a gradual semantics, and $s$ be the Shapley measure of Definition \ref{def:shapley_measure}. 
The Shapley-based impact measure $\nameImpSI$ is $\impSI{\sigma}{\AS}(X,a) = \sum_{x \in X} \impSI{\sigma}{\AS}(\{x\},a)$, where:
\begin{multline*}
 \impSI{\sigma}{\AS}(\{x\},a) = \\  \Bigg( \Bigg. \sum_{(a_1, \dots, a_n) \in P_E(x,a)}\prod_{ 1\leq i\leq n-1 } s((a_{i+1}, a_i)) - \\
\sum_{(a_1, \dots, a_n) \in P_O(x,a)}\prod_{ 1\leq i\leq n-1 } s((a_{i+1}, a_i)) \Bigg. \Bigg)
\end{multline*}

where $P_O(x,a)$ (resp. $P_E(x,a)$) is the set of all odd (resp. even) paths from $x$ to $a$.
\label{def:shap-impact}
\end{definition}

The following theorem shows that the Shapley-based impact measure of all the arguments in the AF on a given argument will necessarily be negative or neutral.

\begin{theorem}
For any gradual semantics $\sigma$ that satisfies Attack Removal Monotonicity, $\AS = (\A,\C)$, and for any $x \in \A$, $\impSI{\sigma}{\AS}(\A,x) \in [-1, 0 ]$.
\label{theorem:limits}
\end{theorem}

\begin{proof}
    \input{Proofs/TH1_impact_si_all_arguments_bounded_proof}
\end{proof}

The next property (on gradual semantics) implies that an argument cannot contribute for an absolute Shapley measure that is greater than the degree of the attacker.

\begin{property}[Bounded Loss]
We say that a gradual semantics $\sigma$ satisfies Bounded Loss iff for every $\AS = (\A,\C)$, $s = \Sh(\sigma, \AS)$, every $a,b \in \A$, $(a,b) \in \C$, then $|s((a,b))| \leq \sigma_\AS(a).$
\label{prop:limitation}
\end{property}

\begin{conjecture}
For any $\sigma \in \{ \Hbs, \Car, \Max, \CS\}$, $\sigma$ satisfies Bounded Loss.
\end{conjecture}

\begin{conjecture}
Let $\AS = (\A,\C)$ be an arbitrary AF.
For any gradual semantics $\sigma$ that satisfies Bounded Loss, it holds that for any $a \in \A$ and $X \subseteq \A$, $\impSI{\sigma}{\AS}(X,a) \in [-1, 1 ]$.
\label{conjecture:limits}
\end{conjecture}

The first observation is that when a gradual semantics does not satisfy Property \ref{prop:limitation}, Conjecture \ref{conjecture:limits} is not satisfied as shown by the next example.

\begin{example}
    
Let us consider the AF
represented in Figure \ref{counterexample} and a gradual semantics $\sigma$ such that $\sigma_\AS(a_2) = 0$ and $\sigma_{\AS}(a_1) = 1/6$, meaning that $\sigma$ does not satisfy Property \ref{prop:limitation}.
We have that :
$\impSI{\sigma}{\AS}(\{a2, a6, a4\}, a1)=\impSI{\sigma}{\AS}(\{a2\}, a1) + \impSI{\sigma}{\AS}(\{a6\}, a1) + \impSI{\sigma}{\AS}(\{a4\}, a1) = -5/6 - 0.5*0.5*5/6 - 0.5*0.5*5/6 = -15/12 < -1$

    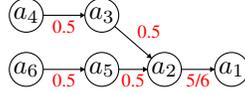
\begin{figure}[!h]
        \centering

        \begin{tikzpicture}[scale=0.075]
\tikzstyle{every node}+=[inner sep=0pt]
\draw [black] (62.1,-29.3) circle (3);
\draw (62.1,-29.3) node {$a_1$};
\draw [black] (50.2,-29.3) circle (3);
\draw (50.2,-29.3) node {$a_2$};
\draw [black] (39.1,-19.7) circle (3);
\draw (39.1,-19.7) node {$a_3$};
\draw [black] (25.7,-19.7) circle (3);
\draw (25.7,-19.7) node {$a_4$};
\draw [black] (25.7,-29.3) circle (3);
\draw (25.7,-29.3) node {$a_6$};
\draw [black] (39.1,-29.3) circle (3);
\draw (39.1,-29.3) node {$a_5$};
\draw [black] (28.7,-19.7) -- (36.1,-19.7) node[midway,label={[label distance=1] 270:  \color{red} \scriptsize 0.5}]{};
\fill [black] (36.1,-19.7) -- (35.3,-19.2) -- (35.3,-20.2);
\draw [black] (41.37,-21.66) -- (47.93,-27.34) node[midway,label={[label distance=1] 45:  \color{red} \scriptsize 0.5}]{};
\fill [black] (47.93,-27.34) -- (47.65,-26.44) -- (47,-27.19);
\draw [black] (42.1,-29.3) -- (47.2,-29.3) node[midway,label={[label distance=1] 270:  \color{red} \scriptsize 0.5}]{};
\fill [black] (47.2,-29.3) -- (46.4,-28.8) -- (46.4,-29.8);
\draw [black] (28.7,-29.3) -- (36.1,-29.3) node[midway,label={[label distance=1] 270:  \color{red} \scriptsize 0.5}]{};
\fill [black] (36.1,-29.3) -- (35.3,-28.8) -- (35.3,-29.8);
\draw [black] (53.2,-29.3) -- (59.1,-29.3) node[midway,label={[label distance=1] 270:  \color{red} \scriptsize 5/6}]{};
\fill [black] (59.1,-29.3) -- (58.3,-28.8) -- (58.3,-29.8);
\end{tikzpicture}

        \caption{Representation of an AF with 6 arguments. The red number on an edge $c$ represents $s(c)$.}
        \label{counterexample}
    \end{figure}
\end{example}

Shapley-based impact measure is not restricted to acyclic graphs.
Figure \ref{fig:shapley} shows the degree of all arguments for the h-categoriser semantics (in blue) and how the Shapley measure associates with each attack in $\C$, its intensity (in red). Using the Shapley-based impact measure (see Definition \ref{def:shap-impact}), the impact of $\{a_6\}$ on $a_4$ is $0.098$, of $\{a_5\}$ on $a_4$ is $-0.196$, and of $\{a_6, a_5\}$ on $a_4$ is $-0.098$. The impact of $\{a_{10}\}$ on $a_4$ is $ \left(\sum\limits_{i=1}^\infty (-0.382)^{2(i-1) + 1} \right) 0.382 \times 0.235  \simeq -0.0402$. 
Similarly, if we consider a self-attacking argument $a$, it would have a degree of $0.618$, its self-attack would have an intensity of $1 - 0.618 = 0.3820$, and the impact of $\{a\}$ on $a$ would be $\left( \sum\limits_{i=1}^\infty 0.3820^{2i}\right) - \left( \sum\limits_{i=1}^\infty 0.3820^{2(i-1)+1}\right) \simeq -0.276$, with the left (resp. right) element representing the impact of the argument defending (resp. attacking) itself.

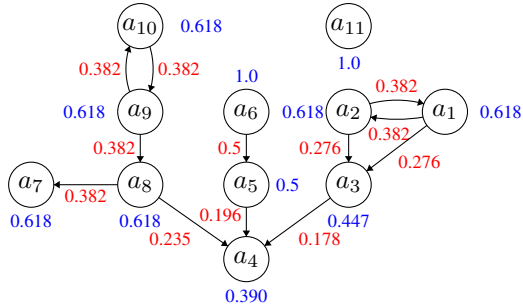
\begin{figure}
    \centering
    
\begin{center}
\begin{tikzpicture}[scale=0.1]
\tikzstyle{every node}+=[inner sep=0pt]
\draw [black] (39.3,-40.4) circle (3);
\draw (39.3,-40.4) node[label={[label distance=8] 270:  \color{blue} \scriptsize  0.390}] {$a_4$};
\draw [black] (52.9,-30.5) circle (3);
\draw (52.9,-30.5) node[label={[label distance=8] 270:  \color{blue} \scriptsize  0.447}] {$a_3$};
\draw [black] (25.2,-30.5) circle (3);
\draw (25.2,-30.5) node[label={[label distance=8] 270:  \color{blue} \scriptsize  0.618}] {$a_8$};
\draw [black] (39.3,-30.5) circle (3);
\draw (39.3,-30.5) node[label={[label distance=6] 0:  \color{blue} \scriptsize  0.5 }] {$a_5$};
\draw [black] (39.3,-20.8) circle (3);
\draw (39.3,-20.8) node[label={[label distance=8] 90:  \color{blue} \scriptsize  1.0 }] {$a_6$};
\draw [black] (25.2,-20.8) circle (3);
\draw (25.2,-20.8) node[label={[label distance=8] 180:  \color{blue} \scriptsize  0.618}] {$a_9$};
\draw [black] (52.9,-20.8) circle (3);
\draw (52.9,-20.8) node[label={[label distance=4] 180:  \color{blue} \scriptsize  0.618}] {$a_2$};
\draw [black] (65.7,-20.8) circle (3);
\draw (65.7,-20.8) node[label={[label distance=8] 0:  \color{blue} \scriptsize  0.618}] {$a_1$};
\draw [black] (25.2,-9.4) circle (3);
\draw (25.2,-9.4) node[label={[label distance=8] 1800:  \color{blue} \scriptsize  0.618}] {$a_{10}$};
\draw [black] (10.7,-30.5) circle (3);
\draw (10.7,-30.5) node[label={[label distance=8] 270:  \color{blue} \scriptsize  0.618}] {$a_7$};
\draw [black] (52.9,-9.4) circle (3);
\draw (52.9,-9.4) node[label={[label distance=8] 270:  \color{blue} \scriptsize  1.0}] {$a_{11}$};
\draw [black] (55.693,-19.722) arc (104.95602:75.04398:13.975) node[midway,label={[label distance=2] 90:  \color{red} \scriptsize  0.382}]{};
\fill [black] (62.91,-19.72) -- (62.26,-19.03) -- (62,-20);
\draw [black] (62.788,-21.509) arc (-80.42236:-99.57764:20.961) node[midway,label={[label distance=2] 270:  \color{red} \scriptsize  \!\!\!\!\!0.382}]{};
\fill [black] (55.81,-21.51) -- (56.52,-22.13) -- (56.68,-21.15);
\draw [black] (52.9,-23.8) -- (52.9,-27.5) node[midway,label={[label distance=2] 180:  \color{red} \scriptsize  0.276}]{};
\fill [black] (52.9,-27.5) -- (53.4,-26.7) -- (52.4,-26.7);
\draw [black] (63.31,-22.61) -- (55.29,-28.69) node[midway,label={[label distance=2] 280:  \color{red} \scriptsize  0.276}]{};
\fill [black] (55.29,-28.69) -- (56.23,-28.6) -- (55.63,-27.81);
\draw [black] (50.47,-32.27) -- (41.73,-38.63) node[midway,label={[label distance=3] 275:  \color{red} \scriptsize 0.178}]{};
\fill [black] (41.73,-38.63) -- (42.67,-38.57) -- (42.08,-37.76);
\draw [black] (27.66,-32.22) -- (36.84,-38.68)node[midway,label={[label distance=3] 260:  \color{red} \scriptsize 0.235}]{};
\fill [black] (36.84,-38.68) -- (36.48,-37.81) -- (35.9,-38.63);
\draw [black] (39.3,-33.5) -- (39.3,-37.4) node[midway,label={[label distance=1] 175:  \color{red} \scriptsize 0.196}]{};
\fill [black] (39.3,-37.4) -- (39.8,-36.6) -- (38.8,-36.6);
\draw [black] (39.3,-23.8) -- (39.3,-27.5) node[midway,label={[label distance=1] 180:  \color{red} \scriptsize 0.5}]{};
\fill [black] (39.3,-27.5) -- (39.8,-26.7) -- (38.8,-26.7);
\draw [black] (25.2,-23.8) -- (25.2,-27.5) node[midway,label={[label distance=1] 180:  \color{red} \scriptsize 0.382}]{};
\fill [black] (25.2,-27.5) -- (25.7,-26.7) -- (24.7,-26.7);
\draw [black] (23.8,-18.161) arc (-161.13492:-198.86508:9.466) node[midway,label={[label distance=1] 180:  \color{red} \scriptsize 0.382}]{};
\fill [black] (23.8,-12.04) -- (23.07,-12.63) -- (24.01,-12.96);
\draw [black] (26.467,-12.108) arc (16.77546:-16.77546:10.367) node[midway,label={[label distance=1] 0:  \color{red} \scriptsize 0.382}]{};
\fill [black] (26.47,-18.09) -- (27.18,-17.47) -- (26.22,-17.18);
\draw [black] (22.2,-30.5) -- (13.7,-30.5) node[midway,label={[label distance=1] 270:  \color{red} \scriptsize 0.382}]{};
\fill [black] (13.7,-30.5) -- (14.5,-31) -- (14.5,-30);
\end{tikzpicture}
\end{center}
    \caption{Intensity of attacks with Shapley measure for the h-categoriser semantics. Values in blue are the degrees of the arguments whereas the values in red represent the intensity of the attacks.}
    \label{fig:shapley}
\end{figure}

\subsection{Observations}
\label{sec:obs-comp}

In Table \ref{tab:ex-impact}, we provide the impact returned by $\nameImpDV$ and $\nameImpSI$ of several sets of arguments on $a_4$ in the AF represented in Figure \ref{fig:AF}. 
Since those two approaches follow different intuitions, we can make several observations. 

For example, we can see that, for $\sigma \in \{ \Hbs, \Car, \Max, \CS\}$, $\impDV{\Sem}{\AS}(\{a_8\},a_4) = \impDV{\Sem}{\AS}(\{a_8,a_{10}\},a_4)$ whereas $\impSI{\Sem}{\AS}(\{a_8\},a_4) \neq \impSI{\Sem}{\AS}(\{a_8,a_{10}\},a_4)$.
The idea is that the Shapley-based impact measure of a set on a target argument can be ``decomposed" as the sum of the impact of each argument of that set on the argument. Since $\impSI{\Sem}{\AS}(\{a_{10}\},a_4) \neq 0 $, it holds that $\impSI{\Sem}{\AS}(\{a_8\},a_4) \neq \impSI{\Sem}{\AS}(\{a_8,a_{10}\},a_4)$. 
In the case of the revised version of $\impDV{}{}$, the impact of a set $X$ on $y$ is the difference in acceptability degree of $y$ between when the direct external attacks on $X$ are removed and when $X$ is removed. 
Here, we have that $\sigma_{\AS \compOpAtt{a_4} \setAttackAtt(\{a_8,a_{10}\})}(a_4) = \sigma_{\AS \compOpAtt{a_4} \setAttackAtt(\{a_8\})}(a_4)$ and $\sigma_{\AS \compOpArg{a_4} \{a_8\}}(a_4) = \sigma_{\AS \compOpArg{a_4}\{a_8,a_{10}\}}(a_4)$, thus the equality.

For the max-based semantics, the impact of some set of arguments (e.g. $\{a_{1}\}$ or $\{a_5,a_6\}$) on $a_4$ is neutral when $\nameImpDV$ is used. 
This is not the case when $\nameImpSI$ as this measure is based on the extended Shapley measure which never assigns a value of 0 to any attacks (for the semantics considered).
Apart from these two cases, we note that the polarity of the impact (i.e., positive or negative) is often the same for our two approaches, given a set of arguments $X$ and a gradual semantics $\Sem$ that satisfies Attack Removal Monotonicity.

Our aim is now to provide a principle-based study to compare these impact measures to explain the common features and the differences observed in this sub-section.

\begin{table}[t]
    \centering
    \resizebox{\linewidth}{!}{
    \begin{tabular}{|c|c|c|c|c|c|c|c|c|}
    \hline
    & \multicolumn{4}{|c|}{$\impDV{\sigma}{\AS}(X,a_4)$}&\multicolumn{4}{|c|}{$\impSI{\sigma}{\AS}(X,a_4)$} \\
    \hline
        $X$ & ${\nameHbs}$ & $\nameCar$ & $\nameMax$& $\nameCS$ & ${\nameHbs}$ & $\nameCar$ & $\nameMax$ & $\nameCS$ \\
        \hline
         $\{ a_1\}$ & 0.015 & 0.001 & 0.0 & 0.072 & 0.036 & 0.056 & 0.019 & 0.026 \\
         $\{ a_1, a_2\}$ & 0.069 & 0.012 & 0.118 & 0.161 & 0.071 & 0.112 & 0.037 & 0.051 \\
         $\{ a_9\}$ & 0.069 & 0.011 & 0.118 & 0.107 & 0.105 & 0.236 & 0.061 & 0.076 \\
         $\{ a_8\}$ & -0.174 & -0.082 & -0.118 & -0.327 & -0.235 & -0.264 & -0.135 & -0.291 \\
         $\{a_{10} \}$ & -0.026 & -0.002 & -0.018 & -0.034 & -0.041 & -0.138 & -0.024 & -0.018\\ 
         $\{ a_8, a_{10}\}$ & -0.174 & -0.082 & -0.118 & -0.327 & -0.276 & -0.402 & -0.159 & -0.309 \\
         $\{ a_8, a_9, a_{10}\}$ & -0.124 & -0.072 & 0.0 & -0.246 & -0.17 & -0.167 & -0.098 & -0.233 \\
         $\{ a_5\}$ & -0.158 & -0.079 & -0.118 & -0.327 & -0.196 & -0.255 & -0.111 & -0.212 \\
         $\{ a_6\}$ & 0.064 & 0.011 & 0.118 & 0.107 & 0.098 & 0.17 & 0.056 & 0.069 \\
         $\{ a_5, a_6\}$ & -0.094 & -0.068 & 0.0 & -0.220 & -0.098 & -0.085 & -0.056 & -0.143 \\
         \hline
    \end{tabular}
    }
    \caption{Impacts of several sets of arguments on $a_4$ using different gradual semantics.}
    \label{tab:ex-impact}
\end{table}

\section{Desirable principles for impact measures}
\label{sec:desirable}

Impact measures are usually defined in a general way and can be paired with any gradual semantics (see e.g.\ ~\cite{DBLP:conf/ecsqaru/DelobelleV19}). 
In the rest of this section, we define the desirable principles of a pairing of an impact measure $\nameImp$ with a gradual semantics $\sigma$, denoted $\nameImp^\sigma$. 
Namely, $\nameImp^\sigma$ takes as input any AF $\AS$ and returns $\imp{\sigma}{\AS}$.

The reason that we do not define the desirable principles at the level of the impact measure itself is that many of the desirable behaviour that we would expect from an impact measure can disappear if the gradual semantics is ill-defined (e.g.\ if it does not satisfy anonymity).
Some of the principles we introduce and use for analysing impact measures under gradual semantics are inspired from the property analysis of gradual semantics~\cite{amgoud_acceptability_2017}, i.e., the Anonymity, Independence, and Directionality properties.

\subsection{List of principles} \label{sec:list_principles}

Unless stated explicitly, all the principles are deﬁned for an impact measure $\nameImp$ and a gradual semantics $\sigma$.

Let us start by introducing the notion of isomorphism between argumentation frameworks.

\begin{definition}
    Let $\AS = (\A,\C)$ and $\AS' = (\A', \C')$ be two AFs. An isomorphism from $\AS$ to $\AS'$ is a bijective function $f$ from $\A$ to $\A'$ such that for all $a,b \in \A$, $(a,b) \in \C$ iff $(f(a), f(b)) \in \C'$.\footnote{For a function $f$ and the set $X$, we use the standard notation $f(X)$ to mean $\{f(x) \mid x \in X\}.$}
    If $\AS = \AS'$, $f$ is called an automorphism.
\end{definition}

Impact Anonymity states that the impact of a set of arguments on an argument should not depend on the names of the arguments.

\begin{principle}[Impact Anonymity]\label{anonymity}
$\imp{\sigma}{}$ satisfies Impact Anonymity iff for any two AFs $\AS = (\A,\C)$, $\AS' = (\A',\C')$, and any isomorphism $f$ from $\AS$ to $\AS'$, the following holds: 
$\forall X \subseteq \A, a \in \A, \imp{\sigma}{\AS}(X,a) = \imp{\sigma}{\AS'}(f(X),f(a))$.
\end{principle}

Impact Independence states that the impact of a set of arguments $X$ on an argument $a$ should not depend on the arguments which are not connected to $X$ nor to $a$ by a path.

\begin{principle}[Impact Independence]\label{independence}

$\imp{\sigma}{}$ satisfies Impact Independence iff for any two AFs $\AS = (\A,\C)$, $\AS' = (\A',\C')$, where $\A \cap \A' = \emptyset$, the following holds: 
    $\forall X \subseteq \A, a \in \A, \imp{\sigma}{\AS}(X,a) = \imp{\sigma}{\AS \oplus \AS'}(X,a)$.
\end{principle}

Balanced Impact states that the sum of the impact of a set of arguments $X$ and the impact of an argument $x'$ on an argument $a$ should be equal to the impact of the union of $X$ with the set containing only argument $x'$ on $a$.
Please note that this principle is inspired from~\cite{DBLP:conf/ecsqaru/DelobelleV19} but we generalise it to any sets instead of singleton sets.
While this principle makes it easier to explain the impact of a set using the impact of its individual elements, it also prevents the modelisation of complex behaviour such as \textit{accrual}.

\begin{principle}[Balanced Impact]\label{balancedimpact}

$\imp{\sigma}{}$ satisfies Balanced Impact iff for any AF $\AS = (\A,\C)$, the following holds:
    $\forall X \subseteq \A, x' \in \A \setminus X, a \in \A,$ $\imp{\sigma}{\AS}(X,a) + \imp{\sigma}{\AS}(\{x'\},a) = \imp{\sigma}{\AS}(X \cup \{x'\},a)$.
\label{prop:balanced-impact}
\end{principle}

Void Impact states that an empty set of arguments has no impact on an argument.
\begin{principle}[Void Impact]\label{voidimpact}

$\imp{\sigma}{}$ satisfies Void Impact iff for any AF $\AS = (\A,\C)$, any $a \in \A$, $\imp{\sigma}{\AS}(\emptyset,a) = 0$.
\end{principle}

Impact Directionality states that the impact of a set of arguments $X$ on an argument $a$ remains unchanged in the case of adding an attack in which the target argument is not connected to $a$ via a path.
\begin{principle}[Impact Directionality]\label{directionality}

$\imp{\sigma}{}$ satisfies Impact Directionality iff for any two AFs $\AS = (\A,\C)$ and $\AS' = (\A,\C \cup \{(b,x)\})$, for any $a \in \A$, if there is no path from $x$ to $a$, then for all $X \subseteq \A$,
 $\imp{\sigma}{\AS}(X,a) = \imp{\sigma}{\AS'}(X,a)$.
\end{principle}

Impact Minimisation captures the fact that the impact of a set of arguments $X$ on an argument $a$ can be reduced to a minimal subset of $X$ from which arguments with no path to $a$ have been removed.

\begin{principle}[Impact Minimisation]\label{minimisation}

$\imp{\sigma}{}$ satisfies Impact Minimisation iff for any AF $\AS = (\A,\C)$, any $X \subseteq \A, x' \in X$ such that there is no path from $x'$ to $a$, and $a \in \A$, it holds that $\imp{\sigma}{\AS}(X,a) = \imp{\sigma}{\AS}(X \setminus \{x'\},a)$.
\end{principle}

Zero Impact states that the impact of an argument $x$ on an argument $a$ is zero if $x$ is not connected to $a$ by a path.
\begin{principle}[Zero Impact]\label{zeroimpact}

$\imp{\sigma}{}$ satisfies Zero Impact iff for any AF $\AS = (\A,\C)$ and $a,x \in \A$, if there is no path from $x$ to $a$ then $\imp{\sigma}{\AS}(\{x\},a) = 0$.

\end{principle}

Impact Symmetry says that if there is an automorphism between the attack structures of two arguments $a$ and $b$, then the impact of a set of arguments $X$ on $a$ is the same as the impact of the set containing the image of each argument of $X$ on $b$.

\begin{principle}[Impact Symmetry]\label{symmetry}

$\imp{\sigma}{}$ satisfies Impact Symmetry iff for any AF $\AS = (\A,\C)$, any $a,b \in \A$, the following holds: if $f$ is an automorphism from $ \AS\vert_{\Str(a) \cup \Str(b)}$ to $\AS\vert_{\Str(a) \cup \Str(b)}$ such that $f(a) = b$ and $f(b)=a$, then for all $X \subseteq \A, \imp{\sigma}{\AS}(X,a) = \imp{\sigma}{\AS}(f(X \cap (\Str(a) \cup \Str(b)) ),b)$.
\end{principle}

\begin{example}
    
We illustrate Impact Symmetry with Figure \ref{fig:prop8}. Note that $\Str(a) \cup \Str(b) = \{x,y,z,a,b\}$. Assume that $f(a)=b, f(b)=a, f(y)=y, f(x)=z$ and $f(z)=x$. This principle states that $\imp{\sigma}{\AS}(\{x,y\},a)  = \imp{\sigma}{\AS}(\{y,z\},b)$.

    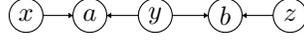
\begin{figure}
        \centering
       \begin{tikzpicture}[scale=0.075]
\tikzstyle{every node}+=[inner sep=0pt]
\draw [black] (11.9,-24.9) circle (3);
\draw (11.9,-24.9) node {$x$};
\draw [black] (23.2,-24.9) circle (3);
\draw (23.2,-24.9) node {$a$};
\draw [black] (34.8,-24.9) circle (3);
\draw (34.8,-24.9) node {$y$};
\draw [black] (47.2,-24.9) circle (3);
\draw (47.2,-24.9) node {$b$};
\draw [black] (58.7,-24.9) circle (3);
\draw (58.7,-24.9) node {$z$};
\draw [black] (14.9,-24.9) -- (20.2,-24.9);
\fill [black] (20.2,-24.9) -- (19.4,-24.4) -- (19.4,-25.4);
\draw [black] (31.8,-24.9) -- (26.2,-24.9);
\fill [black] (26.2,-24.9) -- (27,-25.4) -- (27,-24.4);
\draw [black] (37.8,-24.9) -- (44.2,-24.9);
\fill [black] (44.2,-24.9) -- (43.4,-24.4) -- (43.4,-25.4);
\draw [black] (55.7,-24.9) -- (50.2,-24.9);
\fill [black] (50.2,-24.9) -- (51,-25.4) -- (51,-24.4);
\end{tikzpicture}

        \caption{Illustration of Impact Symmetry}
        \label{fig:prop8}
    \end{figure}
\end{example}

The next principle states that whenever an argument's final strength differs from $1$, there is a set of arguments whose impact explain this difference. This principle is inspired from \cite{DBLP:journals/corr/abs-2401-08879} and ensures that impact measures that always return $0$ are not desirable.
Without loss of generality, this principle can be easily generalised to gradual semantics which maximal value is not $1$.

\begin{principle}[Impact Existence]\label{contribution_existence}

$\imp{\sigma}{}$ satisfies Impact Existence iff for any AF $\AS = (\A,\C)$, any $a \in \A$, the following holds: if $\sigma(a) \neq 1$ then there exists $X \subseteq A$ such that $\imp{\sigma}{\AS}(X,a) \neq 0$.
\end{principle}

\subsection{Links between principles}

Although most of our principles are independent, some of them are related because they deal with neutral impact where there is no path between the arguments whose impact we want to calculate and the target argument. 
This is the case with Impact Symmetry and Impact Minimisation which follows from some other principles.

\begin{theorem} \label{an+ind+direct+minim=symmetry}
Let $\nameImp$ be an impact measure and $\sigma$ be a gradual semantics. If $\imp{\sigma}{}$ satisfies Impact Anonymity, Impact Directionality, Impact Minimisation, and Impact Independence, then $\imp{\sigma}{}$ satisfies Impact Symmetry.
\end{theorem}

\begin{proof}
\input{Proofs/TH2_anonymoty+directionality+minimality+independence_to_symmetry}
\end{proof}

\begin{theorem}\label{balancedimpact+zeroimpactSOminimisation}
    Let $\nameImp$ be an impact measure and $\sigma$ be a gradual semantics. If $\imp{\sigma}{}$ satisfies Zero Impact and Balanced Impact, then $\imp{\sigma}{}$ satisfies Impact Minimisation.
\end{theorem}

\begin{theorem}\label{zero_impact_link}
    Let $\nameImp$ be an impact measure and $\sigma$ be a gradual semantics. If $\imp{\sigma}{}$ satisfies Void Impact and Impact Minimisation, then $\imp{\sigma}{}$ satisfies Zero Impact.
\end{theorem}

\begin{proof}
\input{Proofs/TH3_zeroImpact+Balanced_to_minimisation}
\end{proof}

As stated at the beginning of this section, defining principles solely at the level of impact measure is not relevant, as the calculation is based on the use of scores returned by gradual semantics.
A number of axiomatic studies have been carried out on these semantics in order to guarantee correct behaviour, but also to better understand the results returned.
In this way, it is possible to use the fact that certain properties satisfied by gradual semantics also guarantee the satisfaction of certain principles for $\impDV{\sigma}{}$ and $\impSI{\sigma}{}$.

\begin{proposition}\label{SI_independence_then_impact_independence}
   $\impSI{\Sem}{}$ and $\impDV{\Sem}{}$ satisfy Impact Independence for any gradual semantics $\Sem$ that satisfies Independence.
\end{proposition}

\begin{proof}
\input{Proofs/Prop_Imp_SI_Hbs_Car_Max_Independence_then_Impact_Independence}
\end{proof}

\begin{proof}
\input{Proofs/Prop_Imp_DV_Hbs_Car_Max_Independence_then_Impact_Independence}
\end{proof}

\begin{proposition}\label{SI_directionality_then_impact_directionality}
    $\impSI{\Sem}{}$ and $\impDV{\Sem}{}$ satisfy Impact Directionality for any gradual semantics $\Sem$ that satisfies Directionality.
\end{proposition}

\begin{proof}
\input{Proofs/Prop_Imp_SI_Hbs_Car_Max_directionality_then_Impact_Directionality}
\end{proof}

\begin{proof}
\input{Proofs/Prop_Imp_DV_Hbs_Car_Max_Directionaliy_then_Impact_Diretionality}
\end{proof}

\begin{proposition}\label{zeroimpact_DV_Hbs_Car_Max}
   $\impDV{\Sem}{}$ satisfies Zero Impact for any gradual semantics $\Sem$ that satisfies Directionality and Independence.
\end{proposition}

\begin{proof}
\input{Proofs/Prop_Imp_DV_Independence+Directionalitysemantics_then_zero_impact}
\end{proof}

\begin{proposition}\label{directionalitySOminimisation}
     If $\impDV{\Sem}{}$ satisfies Impact Directionality, then it satisfies Impact Minimisation for any gradual semantics $\Sem$ that satisfies Independence.
\end{proposition}

\begin{proof}
\input{Proofs/Prop_Imp_DV_Hbs_Car_Max_impact_directinality_then_minimization}

\end{proof}

\subsection{Axiomatic evaluation and discussion}

Table \ref{tab:propertiesXimpact} summarizes the results about the axiomatic evaluation of the studied impact measures and gradual semantics.
This axiomatic study only takes into consideration the gradual semantics defined in Section \ref{sec:background}. This choice is motivated by the fact that some of these semantics (such as h-categoriser and CS) have already been studied both axiomatically and in association with an existing impact measure, and the others have only been studied axiomatically but have unique features (such as the use of max to aggregate attacker degree) that we thought would be interesting to study in this work. 
We are aware that there exist other gradual semantics defined in the literature that could have been studied, but the idea was not to study them all but to show that, despite the different behaviors of these semantics, our impact measures can be applied to any gradual semantics.

\begin{theorem} \label{thm:propertiesXimpact}
    The principles of Table \ref{tab:propertiesXimpact} hold.
\end{theorem}

\begin{table*}[t]
    \centering
    \begin{tabular}{|c|c|c|c|c|c|c|c|c|}
    \hline
         Principle&$\impDV{\nameHbs}{}$& $\impSI{\nameHbs}{}$ & $\impDV{\nameMax}{}$ & $\impSI{\nameMax}{}$ &  $\impDV{\nameCar}{}$&  $\impSI{\nameCar}{}$& $\impDV{\nameCS}{}$ & $\impSI{\nameCS}{}$\\
         \hline
         Impact Anonymity & \true & \true & \true & \true & \true & \true & \true  & \true \\
         Impact Independence & \true & \true  & \true  & \true  & \true  & \true  & \false  & \false \\
         Balanced Impact & \false  & \true & \false & \true & \false & \true & \false & \true\\
         Void Impact & \true  & \true & \true  & \true & \true  & \true & \true & \true\\
         Impact Directionality & \true   & \true  & \true   & \true  & \true   & \true  & \false & \false\\
         Impact Minimisation & \true & \true &  \true & \true & \true & \true & \true & \true\\
         Zero Impact & \true & \true  & \true & \true & \true & \true  & \true &\true \\
         Impact Symmetry &  \true & \true &  \true &  \true &  \true & \true & \true & \true\\
         Impact Existence & \true & \true & \true & \true & \true & \true & \true & \true$^{\geq 2}$\\
         \hline
    \end{tabular}
    \caption{Properties satisfied by the studied impact measures ($\nameImpDV$ and $\nameImpSI$) and gradual semantics ($\nameHbs$, $\nameMax$, $\nameCar$, $\nameCS$).
    The symbol \true (resp. \false) means that the property is satisfied (resp. violated) by the impact measure and the gradual semantics considered. The symbol \true$^{\geq 2}$ means that the property is satisfied on the class of graphs where there are at least two arguments with the maximum in-degree.}
    \label{tab:propertiesXimpact}
\end{table*}

We can observe that $\nameImpDV^{\sigma}$ and $\nameImpSI^{\sigma}$ both satisfy Impact Anonymity, Void Impact, Impact Minimisation, Zero Impact and Impact Symmetry for $\sigma \in \{ \Hbs, \Car, \Max, \CS\}$. 

The principles that are not satisfied by $\nameImpDV^{\sigma}$ and $\nameImpSI^{\sigma}$ include Impact Independence and Impact Directionality only for $\nameCS$.
This can be explained by the behaviour of $\nameCS$ that does not satisfy the Independence and the Directionality properties. 
The main axiomatic difference between the two measures concerns the Balanced Impact principle because $\nameImpSI$ satisfies this principle whatever the gradual semantics used in our study, whereas this is never the case for $\nameImpDV$.
The AF depicted in Figure \ref{fig:shapley} shows, for example, that $\impDV{\sigma}{\AS}(\{a_{8},a_{10}\},a_4) \neq \impDV{\sigma}{\AS}(\{a_{8}\},a_4) + \impDV{\sigma}{\AS}(\{a_{10}\},a_4)$ for $\sigma \in \{ \Hbs, \Car, \Max, \CS\}$.
Note also that, although \cite[Proposition 2]{DBLP:conf/ecsqaru/DelobelleV19} states that Balanced Impact is satisfied by the original definition of $\impDV{\CS}{}$ (cf. Definition \ref{def:impactDVold}), this AF can also be used as a counterexample to prove that it is not satisfied.

Moreover, while most of the proofs have been done on general graphs, we have proved the satisfaction of Impact Existence for $\impSI{\CS}{}$ only on the class of graphs where there are at least two arguments with the maximum in-degree, however we conjecture that Impact Existence is also satisfied in the general case.

This principle compliance study can be helpful for choosing which (impact, semantics) pair to use for a specified application. It is interesting to note here that this choice depends on two factors: The satisfiability of the Balanced Impact principle and the satisfiability of the Impact Independence and the Impact Directionality principles.
If an application demands that all principles should be satisfied, then we would use $\nameImpSI$ with one of the three following semantics ($\nameHbs$, $\nameMax$, $\nameCar$).
If an application demands the use of $\nameImpDV$ then we know that, depending on the semantics we choose, we will have Impact Independence and Impact Directionality satisfied (or not).  
Finally, if an application demands the use of the counting semantics, then we know that for both impact measures, the Impact Independence and the Impact Directionality principles are not satisfied while Balanced Impact is satisfied only by $\impSI{}{}$.

\section{Related Work}

\label{sec:related-discussion}

Yin et al.~\cite{yin2023argument} introduced an impact measure to explain the the Discontinuity Free Quantitative Argumentation Debate (DF-QuAD) gradual semantics~\cite{rago_discontinuity-free_2016} in quantitative bipolar argumentation frameworks (QBAFs).
Their impact measure quantifies the contribution of an argument towards topic arguments in QBAFs. Although their work is also inspired by feature attribution explanation methods in machine learning, Yin et al. focus on highlighting the sensitivity of a topic argument's final acceptability degree  w.r.t.\ the other arguments' initial weights. 
Their impact measure is defined only for individual arguments, in acyclic QBAFs and only for the DF-QuAD semantics. However, our two impact measures are both defined for any set of arguments and can be paired with any gradual semantics.
Moreover, the properties they study are explanation-focused, used to assess and characterise their impact measure's ability of providing robust and faithful explanations.
These properties are mostly inspired by properties for machine learning models' explanations such as sensitivity and fidelity. Here, we propose contribution-focused properties, meaning that we evaluate how each pair (impact, semantics) contributes to the final acceptability degree of an argument, w.r.t.\ the argumentation graph's structure. However, we also intend to explore the explanation-focused properties in future work. Namely, we want to study how to produce ``good" explanations for gradual semantics using the impact measures that we defined in this paper.

Kampik et al. \cite{DBLP:journals/corr/abs-2401-08879} propose \textit{contribution functions} and principles in the context of quantitative bipolar argumentation graphs. Contrary to our work, their contribution functions are currently only defined for acyclic graphs and only measure the influence of a single source argument on a topic argument. 
We also note that while they also introduce a Shapley-based contribution, its computation necessitate the addition of the source argument to all possible subgraphs that already contain the topic argument. We argue that this is more computationally expensive than our Shapley-based impact measure which is based on \cite{amgoud_measuring_2017-1}.

The notion of impact for gradual semantics has also been studied by Himeur et al.~\cite{DBLP:conf/sgai/HimeurYBC21}. They measured the impact of agents on arguments in a debate, allowing to identify the agent that is the most influential for a specified argument.   Although the impact measure defined returns the individual impact of an argument on another argument, they defined different aggregation functions that can be used to merge the impact of all the arguments belonging to the same agent, on a particular argument. The impact measure of Himeur et al.~\cite{DBLP:conf/sgai/HimeurYBC21} shares similarities with to the one defined in~\cite{DBLP:conf/ecsqaru/DelobelleV19}. Moreover, their impact measure is studied only for Euler-based semantics~\cite{amgoud2018evaluation} and DF-QuAD w.r.t.\ a set of principles and aggregation functions.

\section{Conclusion and Future Work}
In this paper, we studied the notion of impact of a set of arguments on an argument under gradual semantics. 
For this purpose, we proposed two impact measures: $\impDV{}{}$, a revision of the  measure from \cite{DBLP:conf/ecsqaru/DelobelleV19}, and $\impSI{}{}$, a novel measure based on the Shapley Contribution Measure, which itself is derived from the  measure introduced in \cite{amgoud_measuring_2017-1}.  

We provided a principle-based analysis of these two impact measures under four gradual semantics: h-categoriser ($\nameHbs$), card-based ($\nameCar$), max-based ($\nameMax$) and counting semantics ($\nameCS$). 
For three of the gradual semantics ($\nameHbs$, $\nameMax$, $\nameCar$), we show that $\impSI{\sigma}{}$ satisfies all the principles, while $\impDV{\sigma}{}$ satisfies them all except Balanced Impact.
Concerning $\nameCS$, our two impact measures do not satisfy Impact Independence and Impact Directionality because the two associated properties for gradual semantics are not satisfied by $\nameCS$.

For future work, we plan to study how  to generate explanations for gradual semantics based on these two impact measures.
Providing explanations in abstract argumentation frameworks has been explored for extension-based semantics~\cite{fan2015computing,fan2015explanations,saribatur2020explaining,liao2020explanation,baumann2021choices,borg2021basic,borg_contrastive_2021}. Namely,
Fan and Toni~\cite{fan2015computing} propose to provide explanations in abstract argumentation and assumption-based argumentation frameworks. 
They also propose another method of explanation using the notion of Dispute Trees~\cite{dung2007computing,dung2009assumption}. This work by Fan and Toni was followed by another one that uses dispute trees to provide explanations for arguments which are not acceptable w.r.t.\ the admissible semantics~\cite{fan2015explanations}. 
Borg and Bex~\cite{borg2021basic} propose to compute explanations for credulously/skeptically accepted and non-accepted arguments w.r.t.\ several extension-based semantics, as sets of arguments. 
Borg and Bex~\cite{borg_contrastive_2021} extend this framework to provide contrastive explanations for argumentation-based conclusions in argumentation frameworks derived from an abstract or structured setting. The idea is to explain the acceptance of an argument called the fact and the non-acceptance of a set of arguments called the foil, by returning the elements that caused such acceptance and non-acceptance. 

However, providing explanations in abstract argumentation frameworks has not been explored for gradual semantics.
These results allow us to move to the next step which is to study how we can use 
impact measures to provide explanations for gradual semantics outcomes. 
To generate and deliver explanations tailored to a specific audience,  we want to incorporate insights from the work of Miller~\cite{miller2019explanation}  who analyses various  existing social sciences studies focused on human explicability. 
We also plan to assess our argumentation explanations w.r.t.\ several desirable properties for evaluating explainable AI~\cite{amgoud2022axiomatic,nauta2023anecdotal}.
We also plan to use impact measures to deliver explanations that can help distinguish two semantics satisfying common properties. 
For example, two semantics can allow for the attack to weaken its target, however with different degrees of weakening. 
The current properties are incapable of capturing this difference.

\section*{Acknowledgement}

This work benefited from the support of the project AGGREEY ANR-22-CE23-0005 of the French National Research Agency (ANR).

\newpage

\bibliographystyle{plain}  
\bibliography{arxiv}

\begin{thebibliography}{10}

\bibitem{DBLP:conf/sum/AmgoudB13a}
Leila Amgoud and Jonathan Ben{-}Naim.
\newblock Ranking-based semantics for argumentation frameworks.
\newblock In Weiru Liu, V.~S. Subrahmanian, and Jef Wijsen, editors, {\em
  Scalable Uncertainty Management - 7th International Conference, {SUM} 2013,
  Washington, DC, USA, September 16-18, 2013. Proceedings}, volume 8078 of {\em
  Lecture Notes in Computer Science}, pages 134--147. Springer, 2013.

\bibitem{amgoud2018evaluation}
Leila Amgoud and Jonathan Ben-Naim.
\newblock Evaluation of arguments in weighted bipolar graphs.
\newblock {\em International Journal of Approximate Reasoning}, 99:39--55,
  2018.

\bibitem{amgoud2022axiomatic}
Leila Amgoud and Jonathan Ben-Naim.
\newblock Axiomatic foundations of explainability.
\newblock In {\em 31st International Joint Conference on Artificial
  Intelligence (IJCAI 2022)}, 2022.

\bibitem{amgoud_acceptability_2017}
Leila Amgoud, Jonathan Ben-Naim, Dragan Doder, and Srdjan Vesic.
\newblock Acceptability {Semantics} for {Weighted} {Argumentation}
  {Frameworks}.
\newblock In {\em Proceedings of the {Twenty}-{Sixth} {International} {Joint}
  {Conference} on {Artificial} {Intelligence}, {IJCAI} 2017, {Melbourne},
  {Australia}, {August} 19-25, 2017}, pages 56--62, 2017.

\bibitem{amgoud_measuring_2017-1}
Leila Amgoud, Jonathan Ben-Naim, and Srdjan Vesic.
\newblock Measuring the {Intensity} of {Attacks} in {Argumentation} {Graphs}
  with {Shapley} {Value}.
\newblock In {\em Proceedings of the {Twenty}-{Sixth} {International} {Joint}
  {Conference} on {Artificial} {Intelligence}, {IJCAI} 2017, {Melbourne},
  {Australia}, {August} 19-25, 2017}, pages 63--69, 2017.

\bibitem{amgoud2007unified}
Leila Amgoud, Yannis Dimopoulos, and Pavlos Moraitis.
\newblock A unified and general framework for argumentation-based negotiation.
\newblock In {\em Proceedings of the 6th international joint conference on
  Autonomous agents and multiagent systems}, pages 1--8, 2007.

\bibitem{amgoud_evaluation_2022}
Leila Amgoud, Dragan Doder, and Srdjan Vesic.
\newblock Evaluation of argument strength in attack graphs: {Foundations} and
  semantics.
\newblock {\em Artificial Intelligence}, 302:103607, January 2022.

\bibitem{amgoud2009using}
Leila Amgoud and Henri Prade.
\newblock Using arguments for making and explaining decisions.
\newblock {\em Artificial Intelligence}, 173(3-4):413--436, 2009.

\bibitem{baumann2021choices}
Ringo Baumann and Markus Ulbricht.
\newblock Choices and their consequences-explaining acceptable sets in abstract
  argumentation frameworks.
\newblock In {\em KR}, pages 110--119, 2021.

\bibitem{besnard_logic-based_2001}
Philippe Besnard and Anthony Hunter.
\newblock A logic-based theory of deductive arguments.
\newblock {\em Artif. Intell.}, 128(1-2):203--235, 2001.

\bibitem{BeuselinckDV23}
Vivien Beuselinck, J{\'{e}}r{\^{o}}me Delobelle, and Srdjan Vesic.
\newblock A principle-based account of self-attacking arguments in gradual
  semantics.
\newblock {\em J. Log. Comput.}, 33(2):230--256, 2023.

\bibitem{BonzonDKM23}
Elise Bonzon, J{\'{e}}r{\^{o}}me Delobelle, S{\'{e}}bastien Konieczny, and
  Nicolas Maudet.
\newblock An empirical and axiomatic comparison of ranking-based semantics for
  abstract argumentation.
\newblock {\em J. Appl. Non Class. Logics}, 33(3-4):328--386, 2023.

\bibitem{borg2021basic}
AnneMarie Borg and Floris Bex.
\newblock A basic framework for explanations in argumentation.
\newblock {\em IEEE Intelligent Systems}, 36(2):25--35, 2021.

\bibitem{borg_contrastive_2021}
AnneMarie Borg and Floris Bex.
\newblock Contrastive explanations for argumentation-based conclusions.
\newblock In Piotr Faliszewski, Viviana Mascardi, Catherine Pelachaud, and
  Matthew~E. Taylor, editors, {\em 21st International Conference on Autonomous
  Agents and Multiagent Systems, {AAMAS} 2022}, pages 1551--1553. International
  Foundation for Autonomous Agents and Multiagent Systems {(IFAAMAS)}, 2022.

\bibitem{DBLP:conf/ecsqaru/DelobelleV19}
J{\'{e}}r{\^{o}}me Delobelle and Serena Villata.
\newblock Interpretability of gradual semantics in abstract argumentation.
\newblock In Gabriele Kern{-}Isberner and Zoran Ognjanovic, editors, {\em
  Symbolic and Quantitative Approaches to Reasoning with Uncertainty, 15th
  European Conference, {ECSQARU} 2019, Belgrade, Serbia, September 18-20, 2019,
  Proceedings}, volume 11726 of {\em Lecture Notes in Computer Science}, pages
  27--38. Springer, 2019.

\bibitem{dung_acceptability_1995}
Phan~Minh Dung.
\newblock On the {Acceptability} of {Arguments} and its {Fundamental} {Role} in
  {Nonmonotonic} {Reasoning}, {Logic} {Programming} and n-{Person} {Games}.
\newblock {\em Artif. Intell.}, 77(2):321--358, 1995.

\bibitem{dung2009assumption}
Phan~Minh Dung, Robert~A Kowalski, and Francesca Toni.
\newblock Assumption-based argumentation.
\newblock {\em Argumentation in artificial intelligence}, pages 199--218, 2009.

\bibitem{dung2007computing}
Phan~Minh Dung, Paolo Mancarella, and Francesca Toni.
\newblock Computing ideal sceptical argumentation.
\newblock {\em Artificial Intelligence}, 171(10-15):642--674, 2007.

\bibitem{fan2015computing}
Xiuyi Fan and Francesca Toni.
\newblock On computing explanations in argumentation.
\newblock In {\em Proceedings of the AAAI Conference on Artificial
  Intelligence}, volume~29, 2015.

\bibitem{fan2015explanations}
Xiuyi Fan and Francesca Toni.
\newblock On explanations for non-acceptable arguments.
\newblock In {\em Theory and Applications of Formal Argumentation: Third
  International Workshop, TAFA 2015, Buenos Aires, Argentina, July 25-26, 2015,
  Revised Selected Papers 3}, pages 112--127. Springer, 2015.

\bibitem{gao2016argumentation}
Yang Gao, Francesca Toni, Hao Wang, and Fanjiang Xu.
\newblock Argumentation-based multi-agent decision making with privacy
  preserved.
\newblock In {\em Proceedings of the 2016 International Conference on
  Autonomous Agents \& Multiagent Systems}, pages 1153--1161, 2016.

\bibitem{ha2018designing}
Taehyun Ha, Sangwon Lee, and Sangyeon Kim.
\newblock Designing explainability of an artificial intelligence system.
\newblock In {\em Proceedings of the Technology, Mind, and Society}, pages
  1--1. 2018.

\bibitem{DBLP:conf/sgai/HimeurYBC21}
Areski Himeur, Bruno Yun, Pierre Bisquert, and Madalina Croitoru.
\newblock Assessing the impact of agents in weighted bipolar argumentation
  frameworks.
\newblock In Max Bramer and Richard Ellis, editors, {\em Artificial
  Intelligence {XXXVIII} - 41st {SGAI} International Conference on Artificial
  Intelligence, {AI} 2021, Cambridge, UK, December 14-16, 2021, Proceedings},
  volume 13101 of {\em Lecture Notes in Computer Science}, pages 75--88.
  Springer, 2021.

\bibitem{DBLP:journals/corr/abs-2401-08879}
Timotheus Kampik, Nico Potyka, Xiang Yin, Kristijonas Cyras, and Francesca
  Toni.
\newblock Contribution functions for quantitative bipolar argumentation graphs:
  {A} principle-based analysis.
\newblock {\em CoRR}, abs/2401.08879, 2024.

\bibitem{liao2020explanation}
Beishui Liao and Leendert Van Der~Torre.
\newblock Explanation semantics for abstract argumentation.
\newblock In {\em Computational Models of Argument}, pages 271--282. IOS Press,
  2020.

\bibitem{lundberg2017unified}
Scott~M Lundberg and Su-In Lee.
\newblock A unified approach to interpreting model predictions.
\newblock {\em Advances in neural information processing systems}, 30, 2017.

\bibitem{miller2019explanation}
Tim Miller.
\newblock Explanation in artificial intelligence: Insights from the social
  sciences.
\newblock {\em Artificial intelligence}, 267:1--38, 2019.

\bibitem{muller2012argumentation}
Jann M{\"u}ller and Anthony Hunter.
\newblock An argumentation-based approach for decision making.
\newblock In {\em 2012 IEEE 24th International Conference on Tools with
  Artificial Intelligence}, volume~1, pages 564--571. IEEE, 2012.

\bibitem{nauta2023anecdotal}
Meike Nauta, Jan Trienes, Shreyasi Pathak, Elisa Nguyen, Michelle Peters,
  Yasmin Schmitt, J{\"o}rg Schl{\"o}tterer, Maurice van Keulen, and Christin
  Seifert.
\newblock From anecdotal evidence to quantitative evaluation methods: A
  systematic review on evaluating explainable ai.
\newblock {\em ACM Computing Surveys}, 55(13s):1--42, 2023.

\bibitem{DBLP:conf/ijcai/OrenYVB22}
Nir Oren, Bruno Yun, Srdjan Vesic, and Murilo~S. Baptista.
\newblock Inverse problems for gradual semantics.
\newblock In Luc~De Raedt, editor, {\em Proceedings of the Thirty-First
  International Joint Conference on Artificial Intelligence, {IJCAI} 2022,
  Vienna, Austria, 23-29 July 2022}, pages 2719--2725. ijcai.org, 2022.

\bibitem{potyka2018continuous}
Nico Potyka.
\newblock Continuous dynamical systems for weighted bipolar argumentation.
\newblock In {\em Sixteenth International Conference on Principles of Knowledge
  Representation and Reasoning}, 2018.

\bibitem{PLZL15}
Fuan Pu, Jian Luo, Yulai Zhang, and Guiming Luo.
\newblock Attacker and defender counting approach for abstract argumentation.
\newblock In {\em Proc. of the 37th Annual Meeting of the Cognitive Science
  Society, (CogSci'15)}, 2015.

\bibitem{rago_discontinuity-free_2016}
Antonio Rago, Francesca Toni, Marco Aurisicchio, and Pietro Baroni.
\newblock Discontinuity-{Free} {Decision} {Support} with {Quantitative}
  {Argumentation} {Debates}.
\newblock In {\em Principles of {Knowledge} {Representation} and {Reasoning}:
  {Proceedings} of the {Fifteenth} {International} {Conference}, {KR} 2016,
  {Cape} {Town}, {South} {Africa}, {April} 25-29, 2016.}, pages 63--73, 2016.

\bibitem{ribeiro2016should}
Marco~Tulio Ribeiro, Sameer Singh, and Carlos Guestrin.
\newblock " why should i trust you?" explaining the predictions of any
  classifier.
\newblock In {\em Proceedings of the 22nd ACM SIGKDD international conference
  on knowledge discovery and data mining}, pages 1135--1144, 2016.

\bibitem{rosenfeld2016strategical}
Ariel Rosenfeld and Sarit Kraus.
\newblock Strategical argumentative agent for human persuasion.
\newblock In {\em ECAI 2016}, pages 320--328. IOS Press, 2016.

\bibitem{saribatur2020explaining}
Zeynep~G Saribatur, Johannes~P Wallner, and Stefan Woltran.
\newblock Explaining non-acceptability in abstract argumentation.
\newblock In {\em ECAI 2020}, pages 881--888. IOS Press, 2020.

\bibitem{shapley_value_1953}
LS~Shapley.
\newblock A value for n-person games.
\newblock {\em Contributions to the theory of games}, 2:307--317, 1953.

\bibitem{vilone2021notions}
Giulia Vilone and Luca Longo.
\newblock Notions of explainability and evaluation approaches for explainable
  artificial intelligence.
\newblock {\em Information Fusion}, 76:89--106, 2021.

\bibitem{yin2023argument}
Xiang Yin, Nico Potyka, and Francesca Toni.
\newblock Argument attribution explanations in quantitative bipolar
  argumentation frameworks.
\newblock In Kobi Gal, Ann Now{\'{e}}, Grzegorz~J. Nalepa, Roy Fairstein, and
  Roxana Radulescu, editors, {\em {ECAI} 2023 - 26th European Conference on
  Artificial Intelligence}, volume 372 of {\em Frontiers in Artificial
  Intelligence and Applications}, pages 2898--2905. {IOS} Press, 2023.

\end{thebibliography}

\end{document}